\documentclass{article}
\usepackage{ijcai20}

\usepackage{times}

\usepackage{soul}
\usepackage{url}

\usepackage[small]{caption}
\usepackage{graphicx}
\usepackage{amsmath}
\usepackage{booktabs}
\usepackage{amsthm}
\title{Is the Skip Connection Provable to Reform the Neural Network Loss Landscape?}

\author{
Lifu Wang$^1$ \and
Bo Shen$^1$\footnote{Coresponding Author}\and
Ning Zhao$^{2}$\And
Zhiyuan Zhang$^1$
\affiliations
$^1$Department of Electronic and Information Engineering, Beijing Jiaotong University, Beijing, China\\
Key Laboratory of Communication and Information Systems, Beijing Municipal Commission of Education Beijing Jiaotong University, Beijing, China\\
$^2$Department of Electronic and Information Engineering, Beijing Jiaotong University, Beijing, China\\
State Key Lab of Rail Traffic Control and Safety, Beijing Jiaotong University, Beijing, China
\emails
\{Lifu\_Wang, bshen, n\_zhao, zhangzhiyuan\}@bjtu.edu.cn
}
\usepackage{graphicx}
\usepackage{epstopdf}
\usepackage{mathptmx}      
\usepackage{subfigure}
\usepackage{amsmath}
\allowdisplaybreaks[4]

\usepackage{amsfonts}

\usepackage[ruled]{algorithm2e}

\usepackage{bm}

\usepackage{multirow}
\usepackage{url}
\usepackage{enumerate}
\usepackage{amssymb}
\usepackage{appendix}
\newtheorem{theo}{\textbf{Theorem}}

\newtheorem{lem}[theo]{\textbf{Lemma}}
\newtheorem{ass}{\textbf{Assumption}}

\newtheorem{rem}{\textbf{Remark}}[section]

\begin{document}
\maketitle

\begin{abstract}
The residual network is now one of the most effective structures in deep learning, which utilizes the skip connections to ``guarantee" the performance will not get worse. However, the non-convexity of the neural network makes it unclear whether the skip connections do provably improve the learning ability since the nonlinearity may create many local minima. In some previous works \cite{freeman2016topology}, it is shown that despite the non-convexity, the loss landscape of the two-layer ReLU network has good properties when the number $m$ of hidden nodes is very large. In this paper, we follow this line to study the topology (sub-level sets) of the loss landscape of deep ReLU neural networks with a skip connection and theoretically prove that the skip connection network inherits the good properties of the two-layer network and skip connections can  help to control the connectedness of the sub-level sets, such that any local minima worse than the global minima of some two-layer ReLU network will be very ``shallow". The ``depth" of these local minima are at most $O(m^{(\eta-1)/n})$, where $n$ is the input dimension, $\eta<1$. This provides a theoretical explanation for the effectiveness of the skip connection in deep learning.

\end{abstract}

\section{Introduction}
Although deep learning has achieved great success in almost all the fields of machine learning, understanding the abilities of deep learning theoretically is still a hard problem. A neural network with a large number of hidden nodes has been proved to have strong expressive powers \cite{BarronUniversal}, but the non-convexity makes the model hard to be learned.  The pioneering work in \cite{Krizhevsky2012imageNet} utilized ReLU to improve the performance of deep networks, but ReLU is insufficient to train very deep ones. Resnet \cite{7780459,he2016identity} is the most efficient structure after Alexnet, which utilizes skip connections to let the performance not get worse as the number of the layers increasing, yet due to the non-convexity of the loss function, a rigorous analysis of this property is not easy. There are lots of questions about the effect of this structure. For example, will the skip connections eliminate bad local minima created by the nonlinearity in residual blocks?

\begin{figure}
\centering
\subfigure[Without skip connections]{
\begin{minipage}[c]{0.45\linewidth}
\centering
\includegraphics[width=1.0\textwidth]{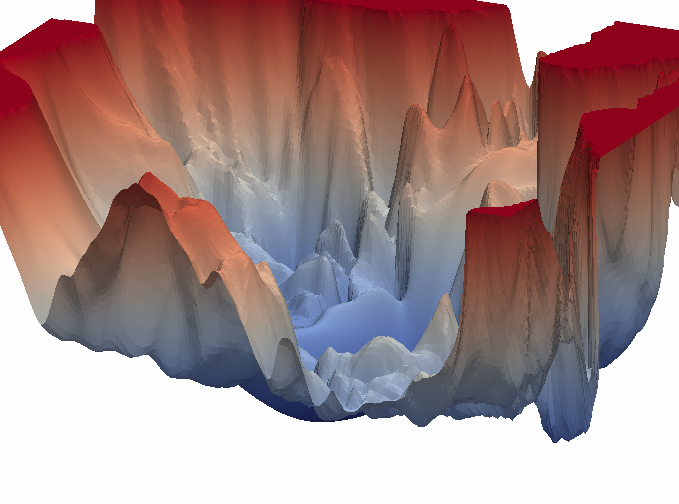}
\end{minipage}%
}%
\subfigure[Skip connections network]{
\begin{minipage}[c]{0.45\linewidth}
\centering
\includegraphics[width=1.0\textwidth]{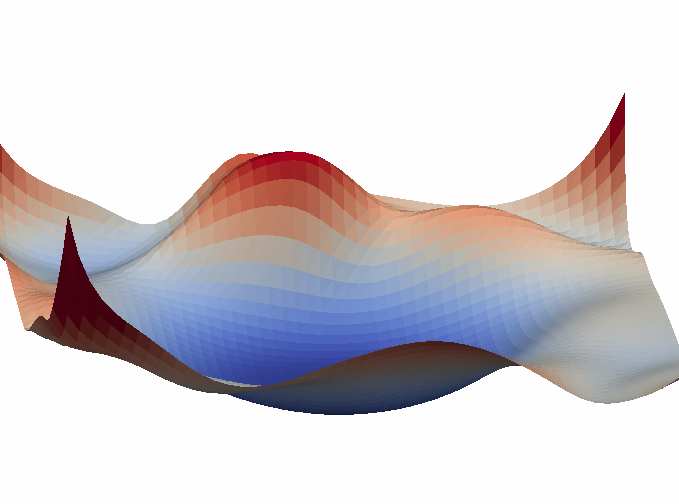}
\end{minipage}%
}%
\centering
\caption{The loss surfaces of ResNet-56 projected into 3d with/without skip connections.\protect\footnotemark[1]}
\end{figure}
\footnotetext[1]{This figure is plotted using ParaView and the visualization method in \cite{NIPS2018_7875}. Their code is available at \url{https://github.com/tomgoldstein/loss-landscape}.}

The problem of local minima is one of the most important questions in the theoretical study of neural networks. Gradient descent based algorithms will stop near the area of saddle points and local minima. By adding noise, it is possible to escape strictly saddle points and shallow local minima \cite{Zhang2017A}, but distinguishing a bad but deep local minimum from a global one can be very hard \cite{jin2018on}. If the loss of the suboptimal local minima created by the nonlinearity and multi-layer structures is very large, the performance of such networks may be very bad. However it has been shown that bad local minima are common in the non-over-parameterized two-layer networks \cite{safran2018spurious}, so that it is generally hard to guarantee there is no bad local minima.
Thus two questions arise naturally:

1.Since the loss in deep learning is not convex, what does the loss landscape of the deep neural network look like?

2.Since the residual network is now very successful, is the deep neural network with skip connections provably good such that more layers will not create more bad local minima?

There are some important works on these two questions. In the study of the landscape of the neural network, one of the most important properties is the ``mode connectivity''\cite{kuditipudi2019explaining} of networks, which has been proved in theory in \cite{freeman2016topology}, and shown by experiments in \cite{article,draxleressentially}. In these works, they showed that the sub-level sets of the loss function are nearly connecting, and the local minima found by gradient descent can be connected by some continuous paths in the parameter space and the loss on the path is almost const, such that the landscape has very good properties. One can guess these properties are true for any neural network, but it is only proved for two-layer ReLU networks in \cite{freeman2016topology}, and it's hard to go beyond this case.

Another important step is taken by the work in \cite{shamir2018are}, where the author compares the two models:
\begin{equation}\label{e1}
y=\textbf{W}^T(x+\textbf{V}g(\theta,x)),
\end{equation}

\begin{equation}\label{e2}
y=\textbf{W}^Tx.
\end{equation}

It is obvious that (\ref{e1}) has stronger  expressive powers than (\ref{e2}). Surprisingly, the work in \cite{shamir2018are} shows that all the local minima(with MSE loss) created by $g(\theta,x)$ in (\ref{e1}) are provably no worse than the global minimum of convex model (\ref{e2}). However, their method is heavily dependent on the convexity of (\ref{e2}), thus it is hard to be generalized to more general cases. For an arbitrary neural network with skip connections, whether the skip connections can eliminate bad local minimum worse than shallower networks is still a open problem.

In the empirical aspect, the work in \cite{NIPS2018_7875} proposed a visualization method and showed that the skip connection does make the landscape smoother, and the loss landscape has nearly convex behaviors. We plot the loss surface of Resnet-56 with and without skip connections in Figure 1, then we can see the ResNet-56 with skip connections does have smoother and better loss landscape. On the other hand, it is shown in \cite{veit2016residual} that after removing the nonlinear parts in residual layer leaving the skip connection only, the performance will not drop too much. Thus one may guess that residual paths have ensemble-like behaviors. However, due to the non-convexity of the neural network, the principle is hard to be analyzed rigorously in theory.

\begin{figure}
  \centering
  \includegraphics[width=.4\textwidth]{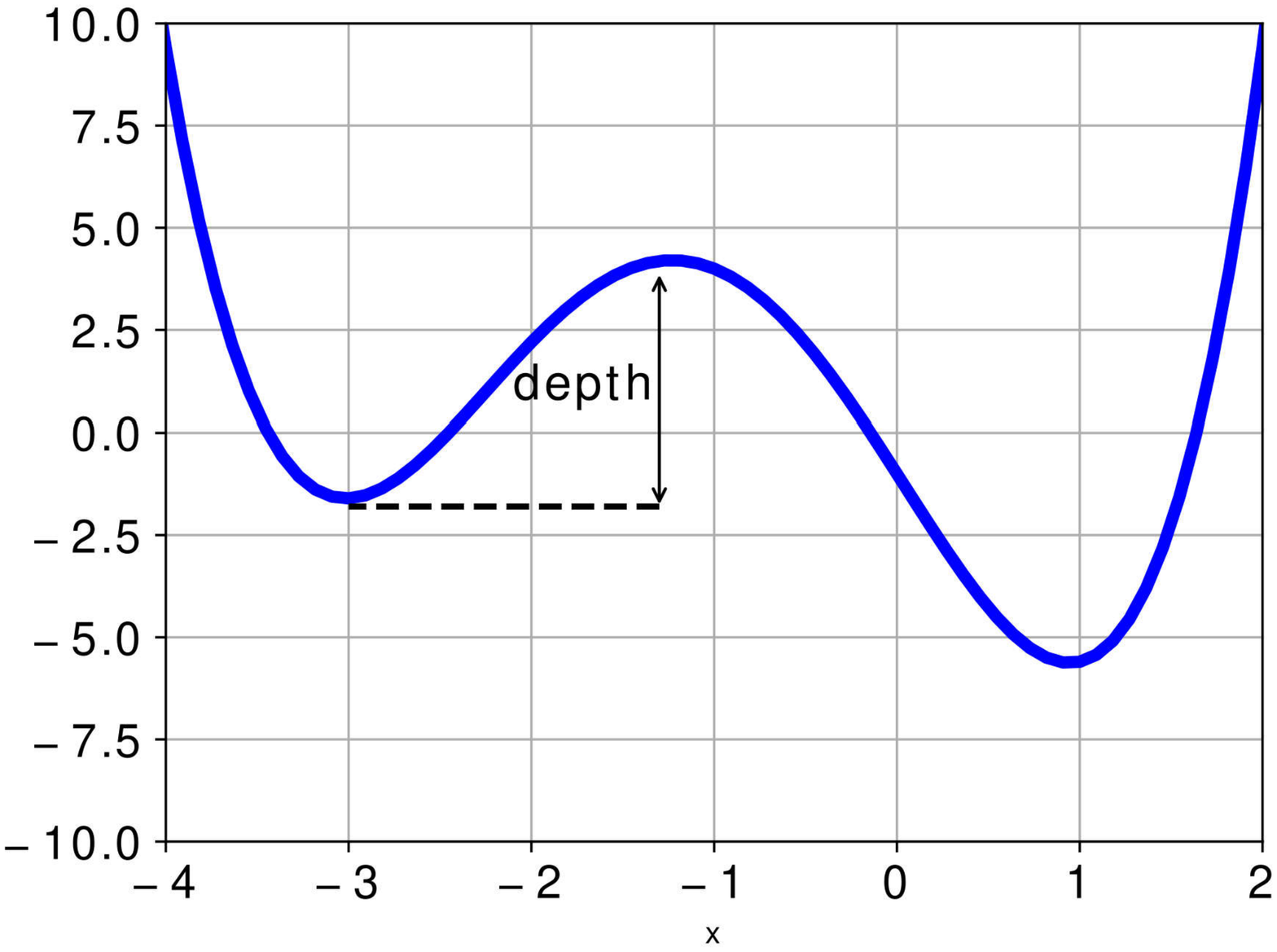}
  \caption{The depth of local minima (defined in Theorem \ref{gap})}
  \label{gap}
\end{figure}

In this work, we study these problems in the perspective of the sub-level sets of the loss landscape as in \cite{freeman2016topology} to estimate the  ``depth" of local minima. In \cite{freeman2016topology}, the authors studied the two-layer ReLU network, and proved that the ``Energy Gap" $\epsilon$ satisfies $\epsilon \approx O(m^{-\frac{1}{n}})$ where $m$ is the width and $n$ is the dimension of the input data, therefore in the large width case, the landscape of the two-layer ReLU network has nearly convex behaviors. Following this line, in this paper, we compare the two networks (We use $\xi$ to denote the parameters of the network, and $\sigma$ is the ReLU activation function):
\begin{equation}\label{ma}
\begin{aligned}
&f_1(\xi)=\textbf{W}_2[\sigma(\textbf{W}_1x)+\textbf{V}_1 g(\theta,\textbf{V}_2\sigma(\textbf{W}_1x))],\\
\end{aligned}
\end{equation}

\begin{equation}\label{ma2}
\begin{aligned}
&f_2(\xi)=\textbf{W}_2\sigma(\textbf{W}_1x),
\end{aligned}
\end{equation}
where $g(\theta,x)$ is a deep neural network with ReLU activation function. The form of $f_1$ is similar to the ``pre-activation" Resnet in \cite{he2016identity}. The structure of these two networks are showed in Figure 3.
\begin{figure}\label{ff3}
\centering
\subfigure{
\begin{minipage}[c]{0.5\linewidth}
\centering
\includegraphics[width=.5\textwidth]{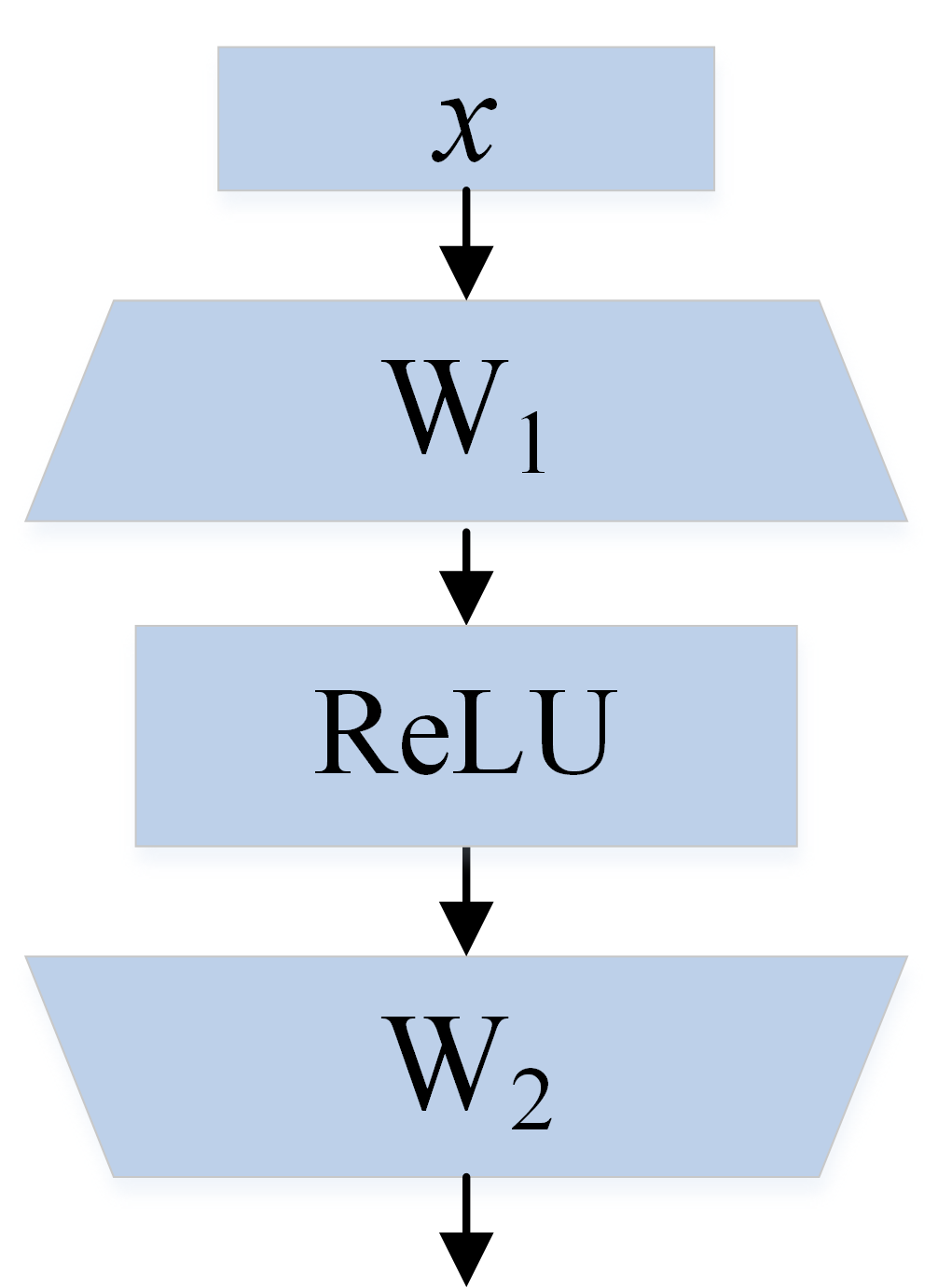}
\end{minipage}%
}%
\subfigure{
\begin{minipage}[c]{0.5\linewidth}
\centering
\includegraphics[width=.36\textwidth]{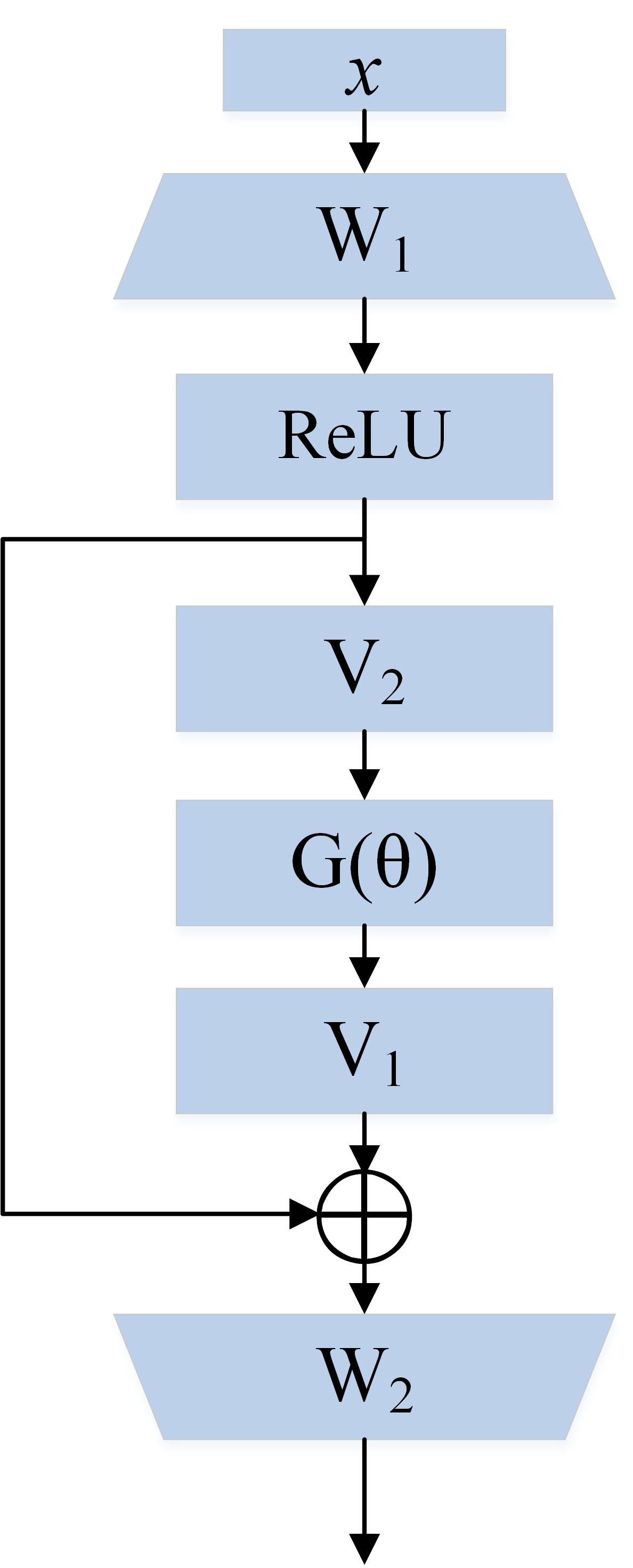}
\end{minipage}%
}%
\centering
\caption{Two-layer network and skip connection network}
\end{figure}

We summarize our main results in the following theorem:

\begin{theo}
(informal)Supposing $X,Y$ are bounded random variables, $L(\xi)$ is a convex function and $R(\xi)$ is the regularization term, and $f_1$ is defined as in (\ref{ma}), $\textbf{W}_1\in \mathbb{R}^{m\times n}$, $\textbf{W}_2\in \mathbb{R}^{d_Y\times m}$, $0<\eta<1$, $l\leq m^{\eta}$, $||\textbf{W}||_0$ is the number of non-zero column vectors in $\textbf{W}$,

\begin{equation}\label{error}
e(l)=\inf_{||\textbf{W}_{2}||_0\leq l,||\textbf{W}_{1,i}||_2= 1} \mathbb{E}\ L(\textbf{W}_2\sigma(\textbf{W}_1x),y) +\kappa\ ||\textbf{W}_2||_1,
\end{equation}

\begin{equation}
F_1(\xi)=\mathbb{E}\ L(f_1(\xi),y)+\kappa R(\xi),
\end{equation}
for every strict local minimum $F^*$ of $F_1$  with $F^*\geq e(l)$, the depth of it is at most $O(m^{\frac{\eta-1}{n}})$.
\end{theo}

This result shows that for a suitable loss function, although $f_1$ is a multi-layer nonlinear network, by virtue of the skip connection, roughly all the local minima worse than the global minimum of the two-layer network  $f_2$ are very ``shallow". The depths of these local minima are  controlled by $\epsilon=O(m^{\frac{\eta-1}{n}})$, so that if $m$ is very large, there is almost no bad strict local minima worse than $e(l)$. From the well known  universal approximation theorem, the expressive power of the two-layer network with ReLU is very strong(this can be easily poved using Hahn-Banach theorem), such that $||\textbf{W}^{*}_2\sigma(\textbf{W}^{*}_1X)-Y||^2\to 0$ as $l\to \infty$ for any functon $Y=f(X)$ under very mild conditions. So this result in fact describes the depth of nearly all the local minima if $m$ is very large.

\section{Related Work}
The global geometry of the deep neural network loss surfaces has been widely concerned for a long time. The characteristics of deep learning are high dimension and non-convex, which make the model hard to be analyzed. The loss surface for deep linear networks was firstly studied in \cite{Kawaguchi2016Deep}. It is shown that all the local minima are global, and all the saddle points are strict. Although the expressive power of the deep linear network is the same as the single-layer one, the loss of the deep linear network is still non-convex. This work pointed out that the linear product of matrices will generally only create saddle points rather than bad local minima. The first rigorous and positive works on non-linear networks are in \cite{tian2017symmetry-breaking} and \cite{du2018when}. In these works, it is shown that, for a  single-hidden-node ReLU network, under a very mild assumption on the input distribution, the loss is one point convex in a very large area. However, for the networks with multi-hidden nodes, the authors in \cite{safran2018spurious} pointed out that spurious local minima are common and indicated that an over-parameterization (the number of hidden nodes should be large) assumption is necessary. The loss surface of the two-layer over-parameterized network was studied in \cite{du2018on} and \cite{MahdiTheoretical}. They showed that the over-parameterization helps two-layer networks to eliminate all the bad local minima, yet their methods required the unrealistic quadratic  activation function (By using Hahn-Banach theorem, it is easy to show there are some functions which cannot be approximated by such networks). A different way to consider the landscape is in  \cite{freeman2016topology}, which studied neural networks with ReLU and required the number of hidden nodes increases exponentially with the dimension of the input. They showed that if the number of the hidden nodes is large enough, the sub-level sets of the loss will be nearly connected.

The loss landscape of nonlinear skip connection networks was studied in \cite{shamir2018are,kawaguchi2018depth,yun2019are}. In these papers, it is shown that all the local minima created by the nonlinearity in the residual layer will never be worse than the linear model, yet their methods heavily rely on the convexity of $l(\textbf{W}x)$  hence very hard to be generalized to more general residual networks. In our work, we focus on more realistic network structure $f=\textbf{W}_2[\sigma(\textbf{W}_1 x)+\textbf{V}_1 g(\theta,\textbf{V}_2\sigma(\textbf{W}_1 x))]$, and give a similar result as \cite{shamir2018are} and extend the work on skip connection network \cite{shamir2018are,kawaguchi2018depth} to more general non-linear cases. The structure we study is similar to the work in \cite{allenzhu2019what}, but we focus on the global geometry of the loss rather than the gradient descent behaviors near neural tangent kernels area. The techniques we use are closely related with the theorems in \cite{freeman2016topology} and our results are much stronger to apply to arbitrarily multilayered residual units with ReLU activation function and a large class of convex functions.

\section{Preliminaries: Connectedness of Sub-level Sets and the Depth of Local Minima}\label{s3}
The loss surface of the model is closely related to the solvability of the optimization problem, and the sub-level set method is a very important tool to study the loss landscape. We consider the risk minimization of the loss:
\begin{equation}
F(\xi)=\mathbb{E}_{X,Y\sim P}\ L(f(X,\xi),Y)+\kappa R(\xi).
\end{equation}
The sub-level set of $F(\xi)$ is defined as:
\begin{equation}
\Omega_F(\lambda)=\{\xi; F(\xi)\leq \lambda\}.
\end{equation}
In the case that $F$ is a convex function, we know that for any $\xi_A,\xi_B$, if $\xi(t)=(1-t)\xi_A+t\xi_B$, we have $F(\xi(t))\leq \max (F(\xi_A),F(\xi_B))$, so the sub-level sets for all $\lambda$ are connected. If $F$ is a function such that all the local minima are global(not need to be convex), for any $\xi_A,\xi_B$ we can find a continuous path $\xi_1(t),\xi_2(t)$ with $F(\xi_1(t)),F(\xi_2(t))$ decreasing, then  $\xi_1(0)=\xi_A$, $\xi_1(1)=\xi^*$, $\xi_2(0)=\xi_B$, $\xi_2(1)=\xi^*$, so that we can produce a path $\xi(t)$ with $F(\xi(t))\leq \max (F(\xi_A),F(\xi_B))$ by splicing the two paths together. Conversely, if the sub-level sets are connected, we can get some information about the strict local minima of the loss function:

\begin{theo}\label{t2}(Proposition 2.1\cite{freeman2016topology})
Supposing  $\Omega_F(\lambda)$ is connected for all $\lambda$, any strict local minimum $\xi^*$, i.e. satisfying that there is a small disk $D=dist(\xi^*,\cdot)\leq \epsilon$ such that for all all the points $\xi'\in D$, $F(\xi')> F(\xi^*)$, is a global minimum.
\end{theo}
Note that this theorem cannot exclude the existence of bad non-strict local minima. Figure \ref{loss} is an example that all the sub-level sets are connected, but bad non-strict local minima exist.
\begin{figure}
  \centering
  \includegraphics[width=.4\textwidth]{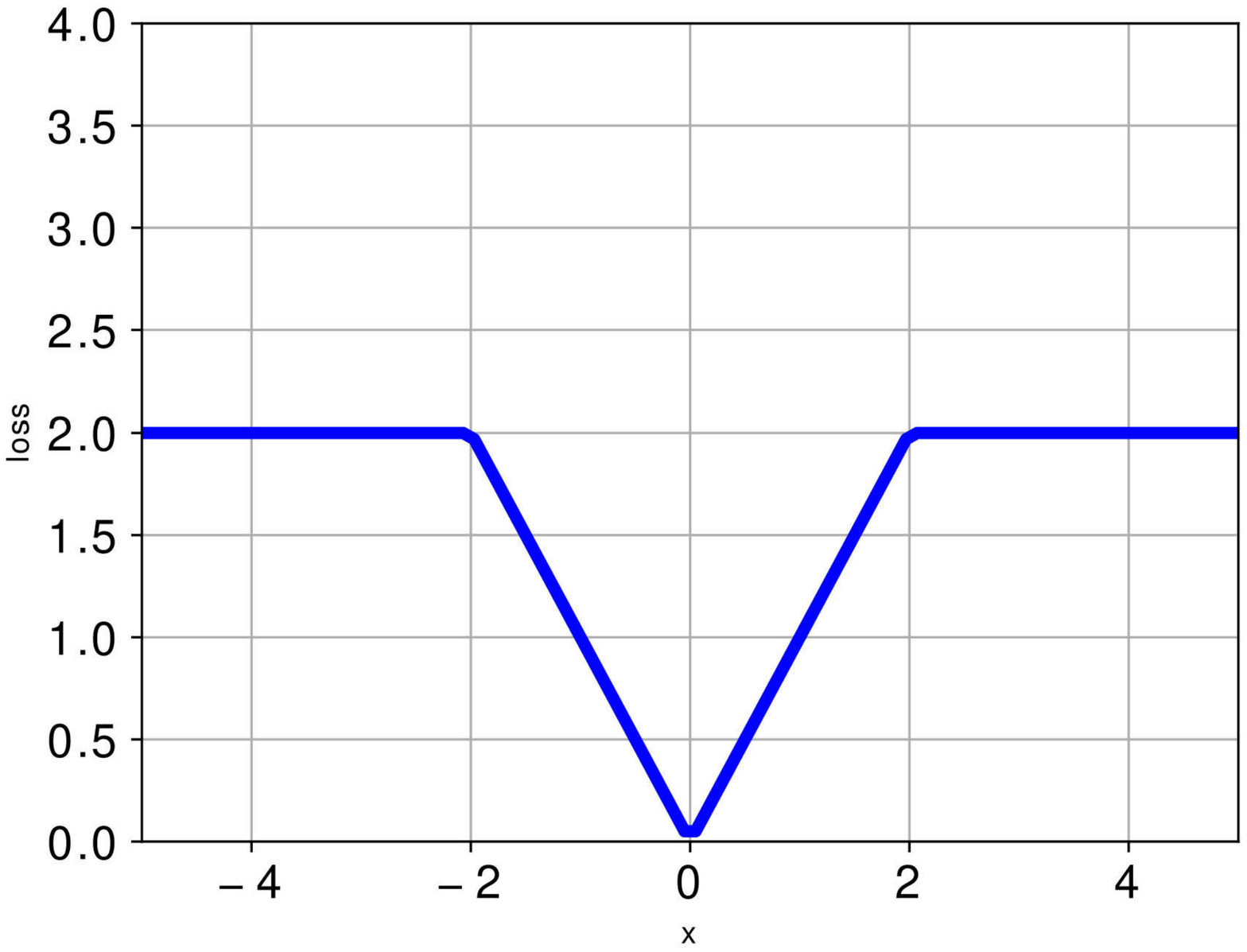}
  \caption{Loss function with non-strict local minima}
  \label{loss}
\end{figure}

In the case not all the sub-level sets are connected, sub-level sets still help us to understand the local minima. In fact we have:
\begin{theo}\label{t3}
Supposing  $\Omega_F(\lambda)$ is connected for all $\lambda \geq \lambda_0$, all the strict local minima $\xi^*$ satisfy $F(\xi^*)\leq \lambda_0$
\end{theo}
The theorem can be proved by showing that there is a decreasing path from any $\xi_A$ to $\xi^*$ where $F(\xi^*)=\lambda_0$. From the condition of this theorem, there is a continuous path $\xi(t)$ from any $\xi_A$ to $\xi^*$ with $F(\xi(t))\leq F(\xi_A)$ and $F(\xi^*)=\lambda_0$. Supposing there is a $t_0>0$ such that on the path $F(\xi(t))$, the part $0\leq t\leq t_0$ is decreasing, due to the fact  $\Omega_F(F(\xi(t_0)))$ is connected, there must be a new path $\xi_1(t)$ such that $F(\xi_1(t)), 0\leq t\leq t_0+\epsilon_0$ is decreasing, and this process can be extended continuously in this way.

There is also a theorem about the depth of local minima:
\begin{theo}\label{gap}
Suppose for any $\xi_A,\xi_B \in \Omega_F(\lambda)$, there is a smallest constant $\epsilon>0$ (the depth)  and a continuous path  $\xi(t)$ connecting $\xi_A$ to $\xi_B$ such that $F(\xi(t))\leq \lambda +\epsilon$. Then for any strict local minimum $\xi^*$ , we have a value $\lambda_1$ and a continuous path $\xi_1(t)$ such that $\xi_1(t)$ connects it to a point $\xi'$, with $\xi'$ not belonging to the same connected component as $\xi^*$ in $\Omega_F(\lambda_1)$ and $\max _t (F(\xi_1(t))-F(\xi^*))= \epsilon$.
\end{theo}
This theorem can be proved by directly constructing such a path as in Theorem \ref{t3} from the conditions of this theorem. It is easy to see that if there is sub-level set $\Omega_F(\lambda_1)$ such that $\theta_A,\theta_B$ are not in the same connected component, there is no decreasing path connecting the two points, so that $\epsilon$ is the depth, which measures the difficulty to jump out a local minimum.

\begin{rem}\label{mr}
As in \cite{freeman2016topology}, the sub-level set is defined to be a closed, thus compact set. Under this definition, completely flat areas (c.f. Figure \ref{loss} ) will not influence the connectedness of such sub-level sets. And when we consider the connectedness of loss sub-level sets, it is sufficient to construct a path outside a zero-measure set. In fact, suppose there is a zero-measure set $S$. By adding small perturbation, there is a continuous path $\gamma(t)$ with $\gamma(t) \in \Theta\setminus S\ $ for almost all $t$, where $\Theta \setminus S$ is the parameter space outside $S$. We suppose $ F(\gamma(t))\leq \lambda +\epsilon$ for almost all $t$ and $\epsilon$ can be arbitrarily small. Due to the continuity of the $F$ and $\bigcap_i (-\infty,\lambda+\epsilon_i]= (-\infty,\lambda]$ where $\epsilon_i$ is a monotone decreasing sequence to 0, we have $F(\gamma(t))\leq \lambda$ for all $t$ and $\gamma(t)$ connects $\xi_A$ and $\xi_B$.
\end{rem}

In the two-layer ReLU network case, the depth of the local minima is given in \cite{freeman2016topology}:

\begin{theo}\label{pt}(Theorem 2.4 in \cite{freeman2016topology})
Consider the loss function $F(\bm{W}_1,\bm{W}_2) =\mathbb{E}_{X,Y\sim P} |Y-\bm{W}_2\sigma(\bm{W}_1X)|^2$, where $X\in \mathbb{R}^{n}, Y\in \mathbb{R}, \bm{W}_1\in \mathbb{R}^{m\times n}, \bm{W}_2\in \mathbb{R}^{1\times m}$, and $\sigma$ is the ReLU activation function. For any $\xi_A,\xi_B \in \Omega_F(\lambda)$, there is a continuous path  $\xi(t)$ connecting $\xi_A$ and $\xi_B$ with $F(\xi(t))\leq \max(\lambda,\epsilon)+O(\alpha)$, where
$$\epsilon=\inf_{l.\alpha} \max (e(l),\delta_{W_1}(m,0,m),\delta_{W_1}(m-l,\alpha,m)),$$ $l=m^{\eta},\alpha=m^{\frac{\eta-1}{n}}, \eta<1$, $e(l)$ is the minimum approximation error using $l$ hidden nodes, $\delta_{W_1}(m-l,\alpha,m)\sim O(\alpha)$, $\delta_{W_1}(m,0,m)\leq \lambda$.
\end{theo}

In the two-layer ReLU case, this shows the depth of all the local minima worse than $e(l)$ is at most  $O(m^{-\frac{1}{n}})$.

\section{Warm up: One-layer Case}

In this section, we consider the loss landscape in the linear case:
\begin{equation}
f(\textbf{W},\textbf{V},\theta,x)=\textbf{W}(x+\textbf{V}g(\theta,x)),
\end{equation}
where $ x\in \mathbb{R}^{d_x}, \textbf{W}\in \mathbb{R}^{d_y\times d_x}, g(\theta,x)\in \mathbb{R}^{d_z}, \textbf{V}\in \mathbb{R}^{d_x\times d_z} $.
with loss
\begin{equation}
F(\textbf{W},\textbf{V},\theta)=\mathbb{E}_{x,y\sim P}\ L(f(\textbf{W},\textbf{V},\theta,x),y).
\end{equation}

In the case y is a scalar and $l$ is the MSE loss function, this has been studied in \cite{shamir2018are}. The result is improved in \cite{kawaguchi2018depth}. And under a weaker condition, we have a new theorem about the sub-level sets:
\begin{theo}\label{t1}
Supposing $w \to \mathbb{E}_{x,y} \ L(wx,y)$ is a function such that the sub-level sets are connected for all $w\in \mathbb{R}^{d_y\times {d_x+d_z}}$ and $\{x\in \mathbb{R}^{d_x+d_z},y\in \mathbb{R}^{d_y} \}$, consider the input $\{x\in \mathbb{R}^{d_x},y \}$ and two models:
\begin{equation}
f_1(\textbf{W})=\textbf{W}(x+\textbf{V}g(\theta,x)),
\end{equation}
\begin{equation}
f_2(\textbf{W})=\textbf{W}x.
\end{equation}
Let $F_1=\mathbb E_{x,y \sim P}\ L(f_1(x),y)$, $F_2=\mathbb E_{x,y\sim P}\ L(f_2(x),y)$.
Assuming $d_y\leq \min \{d_x,d_z\}$,
for any parameter $\theta_A, \theta_B$ and $\lambda \in \mathbb{R}$ with $F(\theta_{\{A,B\}})\leq \lambda$, there exists a continuous path $\gamma(t)$ such that $\gamma(0)=\theta_A$, $\gamma(1)=\theta_B$, and $$F(\gamma(t))\leq \max (\lambda,F^*_{w})$$
where
$$F^*_{w}= \inf_W \  F_2(W ).$$

\end{theo}
\begin{proof}
From lemma \ref{l2}, all the sub-level sets of $f(\textbf{W},\textbf{V})$ are connected. So there is a  path  connecting it to $(\textbf{W}^*,\textbf{V}=0,\theta)$ with the loss bounded by the endpoints. Note that $F^*_{w}= \inf_W F_2(W )=\textbf{W}^*(x+\textbf{0}g(\theta,x))$, our claim follows.
\end{proof}

\begin{lem}\label{l2}
For any distribution $\{x\in \mathbb{R}^{d_x+d_z},y\in \mathbb{R}^{d_y}\}\sim P $ and $d_y\leq \min \{d_x,d_z\}$, supposing all the sub-level sets of function $F(\textbf{Z})=\mathbb{E}_{x,y} \ L(\textbf{Z}x,y)$ are connected, the sub-level sets of function $F(\textbf{W},\textbf{V})=\mathbb{E}_{x,y} \ L([\textbf{W},\textbf{W}\textbf{V}]x,y)$ are also connected.
\end{lem}

\begin{proof}
Let $F(\textbf{Z})=\mathbb{E}_{x,y} \ L(\textbf{Z}x,y)$. For any $\textbf{Z}_1, \textbf{Z}_2$, since the sub-level sets  are connected, there is a continuous path $\textbf{Z}(t)\in \mathbb{R}^{d_y\times (d_x+d_z)}$ such that $\textbf{Z}(0)=\textbf{Z}_1$ and $\textbf{Z}(1)=\textbf{Z}_2$, with $F(\textbf{Z}(t))\leq \max (F(\textbf{Z}_1),F(\textbf{Z}_2))$. To prove the theorem, we need a path with $\textbf{Z}(0)=[\textbf{W}_1,\textbf{W}_1\textbf{V}_1],\textbf{Z}(1)=[\textbf{W}_2,\textbf{W}_2\textbf{V}_2]$ and  $\textbf{Z}(t)=[\textbf{W}(t),\textbf{W}(t)\textbf{V}(t)]$.

Note that $d_y\leq \min \{d_x,d_z\}$ so the sets $rank(\textbf{W}(t))\neq d_y$  have zero measure(since they are closed in Zariski topology). Following the discussion in Remark \ref{mr}, we only need to prove in the case $rank(\textbf{W}(t))=d_y$.  Let $\textbf{Z}_b(t)$ be the last $d_z$ columns of $\textbf{Z}(t)$. We set $\textbf{V}(t)=\textbf{W}^{+}(t)\textbf{Z}_b(t)$, where $\textbf{W}^{+}(t)$ is the Moore-Penrose pseudoinverse of $\textbf{W}(t)$, then $\textbf{V}(t)$ is continuous and  $\textbf{W}(t),\textbf{V}(t)$ is the required path.
\end{proof}

\begin{rem}
In the case $d_y> \min \{d_x,d_z\}$, $\textbf{W}\textbf{V}$ will always be a low rank matrix, so this proof is invalid. However, if the loss function is convex and the eigenvalues of the Hessian matrix are bounded by $a$ and $b$ with $\frac{a}{b}- 1$ small, the sub-level sets will still be connected. This can be proved using the methods in \cite{barber2018gradient,ha2018an}. Since this is not the main target in this paper, we omit it.
\end{rem}

\section{Main Results}
\subsection{Assumptions and Preliminary Lemmas}
\begin{ass}
$L(w,y)$ is a convex function for $w$, $(x,y)\sim P$ are bounded, and $R(\xi)$ is the regularization term. There is a constant $C$ such that all the strict local minima of the loss
\begin{equation}
\begin{aligned}
F(\xi)= &\mathbb{E}_{x,y\sim P} \ L(\textbf{W}_2[\sigma(\textbf{W}_1x)+\textbf{V}_1 g(\theta,\textbf{V}_2\sigma(\textbf{W}_1x))],y)\\
&+\kappa R(\xi)\\
=&\mathbb{E}_{x,y\sim P} \ L(\textbf{W}_2[\sigma(\textbf{W}_1x)+\textbf{V}_1 g(\theta,\textbf{V}_2\sigma(\textbf{W}_1x))],y)\\ &+\kappa(\sum_i(||w_{2,i}||_1+||v_{2,i}||_1)||w_{1,i}||_2+||\textbf{W}_2\textbf{V}_1||_1\\
&+\sum_i ||\theta_i||_F),
\end{aligned}
\end{equation}
satisfying
\begin{equation}
\begin{aligned}
\max &(\sum_i||w_{2,i}||_1||w_{1,i}||_2,\sum_i||v_{2,i}||_1||w_{1,i}||_2,\\
& ||\textbf{W}_2\textbf{V}_1||_1,||\theta_i||_F)\leq C
\end{aligned}
\end{equation}
where $w_{2,i}$, $v_{2,i}$ are the ith column vector of $\textbf{W}_2$ and $\textbf{V}_2$, $w_{1,i}$ is the ith row vector of  $\textbf{W}_1$.
\end{ass}

\begin{ass}
$L$ is locally Lipschitzian:
\begin{equation}
|L(x_1,y)-L(x_2,y)| \leq L_0||x_1-x_2||_F,
\end{equation}
when $||x_1||_F,||x_2||_F\leq C$.
\end{ass}

An example satisfying these assumptions is the MSE loss. In fact we have:
\begin{theo}
Suppose $g(\theta,x)$ is an arbitrary multi-layer ReLU network, with the loss
\begin{equation}
\begin{aligned}
F(\xi)=& \mathbb{E}_{x,y \sim P} \ ||\textbf{W}_2[\sigma(\textbf{W}_1x)+\textbf{V}_1 g(\theta,\textbf{V}_2\sigma(\textbf{W}_1x))]-y||_F^2 \\
&+\kappa R(\xi),
\end{aligned}
\end{equation}
 For any points with $\nabla F=0$, there is a constant C such that these assumptions are satisfied.
\end{theo}

\begin{proof}
It is trivial that Assumption 2 is satisfied. We only need to prove Assumption 1. In this case, all the activation functions in this network are ReLU, so that we can write all the matrix parameter $\textbf{W}$ as $\textbf{W}=t_W\textbf{E}_W$. where $||\textbf{E}_W||_1=1$ and $t_W=||\textbf{W}||_1$, $\sigma(t_W\textbf{E}_W)=t_W\sigma(\textbf{E}_W)$. We fix all  $\textbf{E}_W$ and the loss has the form (We only need to consider the case $d_y=1$ since if $d_y>1$, it can be reduced to $\sum_i|y_i-x_i|^2$):

\begin{equation}
\begin{aligned}
F(t,\theta)=\mathbb{E}_{x,y\sim P}\ &|(t+tv_1v_2\prod_i \theta_i)wx-y|^2\\
+&\kappa(|tw|+\sum_i |\theta_i|+|tv_1|+|tv_2|),
\end{aligned}
\end{equation}
where $t,w,\theta_i$ are the corresponding variables. We have:
\begin{equation}\label{re}
\begin{aligned}
t \nabla_{t}F=&2\mathbb{E}_{x,y\sim P}\ [(t+tv_1v_2\prod_i \theta_i)wx-y][(t+tv_1v_2\prod_i \theta_i)wx]\\
&+\kappa|tw|+\kappa|tv_1|+\kappa|tv_2|\\
=&0.
\end{aligned}
\end{equation}

Note that  $|tw|+|tv_1|+|tv_2|>0$, $\mathbb{E}_x\ [(t+tv_1v_2\prod_i \theta_i)wx-y][(t+tv_1v_2\prod_i \theta_i)wx]<0$ We have
$\mathbb{E}\ [(t+tv_1v_2\prod_i \theta_i)wx]^2 \leq  \mathbb{E}\ |[(t+tv_1v_2\prod_i \theta_i)wx]y|$
so $\kappa(|tw|+|tv_1|+|tv_2|)\sim\mathbb{E}\ |(t+tv_1v_2\prod_i \theta_i)wx|\sim O(\mathbb{E} |y|)$. This proof also applies to other variables, so our claim follows.

\end{proof}

Before proving the main theorem, we need two key lemmas from \cite{freeman2016topology}:
\begin{lem}\label{lll}
Considering a matrix $W\in \mathbb{R}^{m\times n}$, which is equal to give $m$ vectors, and $0<\eta\leq 1$, there is a a collection $Q_m$ of at last $m^{\eta}$ vectors such that for any $v_1, v_2 \in Q_m$,  $\angle v_1, v_2 \leq 2\epsilon_{m,\eta}=2m^{\frac{\eta-1}{n}}$.
\end{lem}
\begin{lem}\label{lll2}
Given $w_1$, $w_2$ with $||w_1||=||w_2||=1$, $\angle w_1,w_2\leq \alpha$, and $\sigma$ is the ReLU activation function, we have $\mathbb{E}_x ||\sigma(w_1x)-\sigma(w_2x)||^2 \leq 4||\Sigma_x|| \alpha^2 $, where $\Sigma_X= \mathbb{E}_{X\sim P}\ XX^T\in \mathbb{R}^{n\times n}$
\end{lem}
These lemmas are from  the proof of corollary 2.5 and proposition 2.3 in \cite{freeman2016topology} respectively.

\subsection{Main Theorem}
\begin{theo}\label{mt}
Consider a distribution $\{x\in \mathbb{R}^{n},y \in \mathbb{R}^{d_y} \}\sim P$, $\sigma$ the ReLU activation function, and a neural network with a skip connection:
\begin{equation}\nonumber
\begin{aligned}
f(\textbf{W}_1,\textbf{W}_2,\textbf{V}_1,\textbf{V}_2,\theta,x)=&\textbf{W}_2[\sigma(\textbf{W}_1x)+\textbf{V}_1 g(\theta,\textbf{V}_2\sigma(\textbf{W}_1x))],
\end{aligned}
\end{equation}
with $L$ a function satisfying assumption 1 and 2. Assume $g(\theta,x)$ is a neural network with ReLU activation functions and $R(\xi)$ is the regular term with $R(\xi)= \sum_i||w_{2,i}||_1||w_{1,i}||_2+\sum_i||v_{2,i}||_1||w_{1,i}||_2+||\textbf{W}_2\textbf{V}_1||_1+\sum_i ||\theta_i||_F$. Let
$$F(\xi)= \mathbb{E}_{x,y\sim P} \ L(f(\xi,x),y) +\kappa R(\xi),$$
then we have: For any $\eta<1, l<m^{\eta}$, $\xi_A, \xi_B$ and $\lambda \in R$ satisfying $F(\xi_{\{A,B\}})\leq \lambda$, there exists a continuous path $\gamma(t)$ such that $\gamma(0)=\xi_A$, $\gamma(1)=\xi_B$, and $$F(\gamma(t))\leq \max(F^*_{w},\lambda)+O(m^{\frac{\eta-1}{n}}),$$
where $m$ is the dimension of $\textbf{W}_1\in R^{m\times n}$ and
\begin{equation}\label{me}\nonumber
\begin{aligned}
F^*_{w}=\inf_{||W_{1,i}||_2=1,||W_2||_0\leq l} \mathbb{E}_{X,Y\sim P} \ l(\textbf{W}_2\sigma(\textbf{W}_1X),Y) +\kappa ||\textbf{W}_2||_1.
\end{aligned}
\end{equation}
\end{theo}

\begin{rem}
Our regularization term is chosen to be compatible with the loss function, and it is also compatible with the two-layer linear network. It can be replaced by using good initialization \cite{glorot2010understanding} or if we only consider a specific bounded area.
\end{rem}
\begin{proof}

Supposing $\textbf{W}_1\in \mathbb{R}^{m\times n}, \textbf{W}_2\in \mathbb{R}^{d_y\times m}, \textbf{V}_2\in \mathbb{R}^{d_g\times m}, \textbf{V}_1\in \mathbb{R}^{n\times d_o}$, we need to construct a path $\gamma(t)$ from $(\textbf{W}_{1,A},\textbf{W}_{2,A},\textbf{V}_{1,A}, \textbf{V}_{2,A},\theta_A)$ to $(\textbf{W}_{1,B},\textbf{W}_{2,B},\textbf{V}_{1,B}, \textbf{V}_{2,B}, \theta_B)$. Note that we only need to construct a path $\gamma_1(t)$ from any $(\textbf{W}_{1,A},\textbf{W}_{2,A},\textbf{V}_{1,A}, \textbf{V}_{2,A}, \theta_A)$ to $(\textbf{W}^*_{1},\textbf{W}^*_{2},\textbf{V}^*_{1}, \textbf{V}^*_{2}, \theta^*)$ with $F(\textbf{W}^*_{1},\textbf{W}^*_{2},\textbf{V}^*_{1}, \textbf{V}^*_{2},\theta^*)=F^*$ and show that $F(\gamma_1(t))\leq \max (F(\gamma_1(0)),F^*)+O(m^{\frac{\eta-1}{n}})$ because the second half of the path $\gamma(t)$ is the inverse of the first half. So we need to construct the following parts:

1.$(\textbf{W}_{1,A},\textbf{W}_{2,A},\textbf{V}_{1,A}, \textbf{V}_{2,A}, \theta_A)$ to $(\textbf{W}_{1,l},\textbf{W}_{2,l},\textbf{V}_{1,l},\textbf{V}_{2,l},\theta_l)$.

On this path, the norms of all matrices are reduced without increasing the loss in virtue of the regularization term, such that the $g(\theta,x)$ and $||\textbf{W}_2\textbf{V}_1||$ are bounded.

2. $(\textbf{W}_{1,l},\textbf{W}_{2,l},\textbf{V}_{1,l}, \textbf{V}_{2,l}, \theta_l)$ to $(\textbf{W}_{1,l},\textbf{W}_{2,s},\textbf{V}_{1,s}, \textbf{V}_{2,s}, \theta_l)$.

On this path, $\textbf{W}_{2,s}$ is a $(m-l)$-term approximation using perturbed atoms to minimize $\mathbb {E}_x||\textbf{W}_{2_s}\sigma(\textbf{W}_{1,l}x)-\textbf{W}_{2,l}\sigma(\textbf{W}_{1,l}x)||_F$, and  $\textbf{V}_{2,s}$ is a $(m-l)$-term approximation to minimize $\mathbb {E}_x||\textbf{V}_{2,s}\sigma(\textbf{W}_{1,l}x)-\textbf{V}_{2,l}\sigma(\textbf{W}_{1,l}x)||^2_F$. The loss increasing along this path is roughly bounded by $O(m^{\frac{\eta-1}{n}})$.

3. $(\textbf{W}_{1,l},\textbf{W}_{2,s},\textbf{V}_{1,s}, \textbf{V}_{2,s}, \theta_l)$ to $(\textbf{W}^*_{1},\textbf{W}^*_{2},\textbf{0}, \textbf{0}, \textbf{0})$.

On this path, $(\textbf{W}^*_{1},\textbf{W}^*_{2})$ are the parameters of l-term approximation:
\begin{equation}
\begin{aligned}
(\textbf{W}^*_{1},\textbf{W}^*_{2})&=argmin_{||w_{1,i}||_2= 1, ||\textbf{W}_{2}||_0\leq l}\\
&\mathbb{E}_{x,y\sim P} \ L(Y,\textbf{W}_2\sigma(\textbf{W}_1x))+\kappa ||\textbf{W}_2||_1.
\end{aligned}
\end{equation}
$\textbf{W}_{2,s}$ is $m-l$ sparse with $l$ zero columns and $\textbf{W}^*_{2}$ has no-zero values only on these l columns. The loss along that path will be upper bounded by $\lambda$ and $e(l)$.

The construction of these path is described as follow:

Step 1: Consider the loss:
\begin{equation}
\begin{aligned}
\mathbb{E}\ &L(Y,\textbf{W}_1[\sigma (\textbf{W}_1 x)+\textbf{V}_1g(\theta,\textbf{V}_2\sigma(\textbf{W}_1 x),x)])\\
&+R(\theta,\textbf{W}_1,\textbf{W}_2,\textbf{V}_1,\textbf{V}_2).
\end{aligned}
\end{equation}

Since the Assumption 1 is satisfied, there is a constant $C$ independent of $m$ and a continuous path without increasing the loss such that
\begin{equation}
\begin{aligned}
\max &(\sum_i||w_{2,i}||_1||w_{1,i}||_2, \sum_i||v_{2,i}||_1||w_{1,i}||_2,\\
& ||\textbf{W}_2\textbf{V}_1||_1, \sum_i ||\theta_i||_F)\leq C,
\end{aligned}
\end{equation}
and $\forall i, ||w_{1,i}||_2 = 1$ where $w_{1,i}$ is the ith row vector of $\textbf{W}_1$. Thus $||g(\theta,x)||_F$ will be $G_0$ Lipschitzian for $x$. Then fix $\textbf{W}_{1}$, $\theta$ and $\textbf{V}_{2}$, and consider the path $\textbf{W}_2(t), \textbf{V}_1(t)$ to the nearest local minimum. The loss will not increase on this path.

Step 2: The $m-l$ spare parameter matrix  $\textbf{W}_{2,s}$ and $\textbf{V}_{2,s}$ are constructed as follow: Find a set $Q_m$ of row vectors in $\textbf{W}_{1,l}$ as in Lemma \ref{lll} and select a vector $v$ at the jth row  such that for any $v_i \in Q_m $ at the ith row of $W_1$, $\angle v_i,v \leq 2\epsilon_{m,\eta}$. $\textbf{W}_{2,s}$ is constrcuted by setting the rows  corresponding to the vectors in $Q_m$ to be zero, and for the corresponding column of $W_{2,s,i}=0, W_{2,s,j}=\sum_i W_{2,l,i} +W_{2,l,j}$. The construction of $\textbf{V}_{2,s}$ is similar. Consider the paths $\textbf{W}_2(t)$,$\textbf{V}_2(t)$ constructed as follow:
$$\textbf{W}_{2}(t)=[\alpha_1,\alpha_2 ... \bar{\alpha}_{i_1},\alpha_{i_1+1}...\widetilde{\alpha}_{j} \bar{\alpha}_{i_2} ...],$$
$$\textbf{V}_{2}(t)=[\beta_1,\beta_2 ... \bar{\beta}_{i_1},\beta_{i_1+1}...\widetilde{\beta}_{j} \bar{\beta}_{i_2} ...].$$
where $i_1,i_2...$ are the indexes corresponding to the rows of vectors in $Q_m$ and $\bar{\alpha}_j, \bar{\beta}_j$ are the ones corresponding to the rows of vectors in the jth row. We have $\bar{\alpha}_{i_1}=(1-t) \alpha_{i_1}$, $\bar{\beta}_{i_1}=(1-t) \beta_{i_1}$ and $\widetilde{\alpha}_{j}=\alpha_{j}+\sum_k t \alpha_{i_k}$, $\widetilde{\beta}_{j}=\beta_{j}+\sum_k  t\beta_{i_k}$. This process will not increase the regularization loss.

For $i\in Q_m$, $j\in Q_2$, let $\sigma(\textbf{W}_{1,l}x)_i=\sigma(\textbf{W}_{1,l}x)_j +n_i$.
\begin{equation}
\begin{aligned}
&||\textbf{W}_2(t)\sigma(\textbf{W}_{1,l} x)-\textbf{W}_{2,s}\sigma(\textbf{W}_{1,l}x)||_F\leq ||\sum_k \alpha_{i_k}n_{i_k}||_F\\
&||\textbf{V}_2(t)\sigma(\textbf{W}_{1,l} x)-\textbf{V}_{2,s}\sigma(\textbf{W}_{1,l}x)||_F\leq ||\sum_k \beta_{i_k}n_{i_k}||_F.
\end{aligned}
\end{equation}

Let $\textbf{V}_{1}(t)= \textbf{W}^+_{2}(t)\textbf{W}_{2}(0)\textbf{V}_{1}(0)$, where $\textbf{W}^+_{2}(t)$ is the Moore-Penrose pseudoinverse. Note that the set $rank\ \textbf{W}_{2}\neq d_y$ has zero measure. We only need to consider the case $rank\ \textbf{W}_{2}(t)= d_y$, then $\textbf{W}_{2}(t)\textbf{V}_{1}(t)=\textbf{W}_{2}(0)\textbf{V}_{1}(0)$. To estimate the loss, note that $L$ is locally Lipschitzian, we have:
\begin{equation}\nonumber
\begin{aligned}
&\mathbb{E} \ |L(f_1(\textbf{V}_{1}(t),\textbf{V}_{2}(t),\textbf{W}_{2}(t),x),y)\\
&-L(f_1(\textbf{V}_{1}(0),\textbf{V}_{2}(0),\textbf{W}_{2}(0),x),y)| +\kappa| R(\xi(t))- R(\xi(0))|\\
&\leq \mathbb{E} \  L_0||\sum_k \alpha_{i_k}n_{i_k}||_F +L_0||\textbf{W}_{2}(0)\textbf{V}_{1}(0)||\cdot G_0||\sum_k \beta_{i_k}n_{i_k}||_F\\
&\sim O(\mathbb{E}_{x,y} \ |n|)\sim O(\epsilon_{m,\eta}).
\end{aligned}
\end{equation}
Step 3: $\textbf{W}_{2,s}$ and $\textbf{V}_{2,s}$ are $m-l$ sparse, so changing the l rows in $\textbf{W}_{1,l}$ will not influence the loss. We consider a path to change these l rows to be same as $\textbf{W}^*_{1}$ and the loss on this path is constant. The second step is to construct a path $(\textbf{W}_{2}(t),\textbf{W}_{1}(t),\textbf{V}_{1}(t))$ with $\textbf{W}_{2}(0) =\textbf{W}_{2,s}, \textbf{W}_{2}(1) =\textbf{W}^*_{2}, \ \textbf{W}_{2}(1)\textbf{V}_{1}(1)=0$. As the proof of Lemma \ref{l2}, since the loss is convex for the final layer, the loss on this path is bounded by the two endpoints, and note that:
\begin{equation}\nonumber
\begin{aligned}
&F(\textbf{W}^*_{1},\textbf{W}^*_{2},\textbf{V}^*_{1},\textbf{V}^*_{2},\theta^*)=F(\textbf{W}^*_{1},\textbf{W}^*_{2},\textbf{V}^*_{1}=\textbf{0},\textbf{V}^*_{2}=\textbf{0},\theta^*=\textbf{0})\\
&\leq \inf_{||w_{1,i}||_2= 1,||\textbf{W}_2||_0\leq l}\mathbb{E}_{X,Y\sim P} \ L(Y,\textbf{W}_2\sigma(\textbf{W}_1 X))+\kappa||\textbf{W}_2||_1.
\end{aligned}
\end{equation}
The properties are satisfied.
\end{proof}

\subsection{Discussion}
Theorem \ref{mt} is a generalization of Theorem \ref{pt} in \cite{freeman2016topology} and the linear case Theorem \ref{t1}. As pointed out in section \ref{s3}, this shows that for all local minima worse than the global minimum $F^*$ of two-layer networks with $l=m^{\eta}$ hidden nodes, the depth is bounded by $O(m^{\frac{\eta-1}{n}})$, so that as $m\to \infty $, $\Omega_F(\lambda)$ is nearly connected if $\lambda>F^*$. Benefitting from the strong expressiveness of the two-layer ReLU network, for almost all the learning problem, $F^*$ in this theorem will be much better than that in the linear case Theorem \ref{t1}.

\section{Conclusion}
In this paper, we studied the loss landscape of the multi-layer nonlinear neural network with a skip connection and consider the connectedness of sub-level sets. The main theorem reveals that by virtue of the skip connection, under mild conditions all the local minima worse than the global minimum of the two-layer ReLU network will be very shallow, such that the ``depths" of these local minima are at most $O(m^{\frac{\eta-1}{n}})$, where $\eta<1$, $m$ is the number of the hidden nodes in the first layer, and $n$ is the dimension of the input data. This result shows that despite the non-convexity of the non-linear networks, skip connections provably help to reform the loss landscape, and in the over-parametrization($m\to \infty$) case, nearly all the strict local minima are no worse than the global minimum of the two-layer ones. Our results provide a theoretical explanation of the effectiveness of the skip connections and take a step to understand the mysterious effectiveness of deep learning networks.

\section*{Acknowledgments}
The authors thanks Peng Zhang and Wenyu Zhang for their great help. This research was funded by funded by the National Key Research and Development Program of China, grant number 2018YFC0831300 and China Postdoctoral Science Foundation (grant number:2018M641172).

\bibliographystyle{named}
\bibliography{References}

\end{document}